\pdfoutput=1

\documentclass{article}

\usepackage{microtype}
\usepackage{graphicx}
\usepackage{subfigure}
\usepackage{booktabs}
\usepackage{amsmath}
\usepackage{amssymb}
\usepackage{amsthm}
\usepackage{multirow}
\usepackage{xcolor}
\usepackage{algorithm}
\usepackage{algorithmic}

\usepackage{hyperref}
\makeatletter
\@ifundefined{theHalgorithm}%
  {}%
  {}
\makeatother

\usepackage[accepted]{mlsys2026}
\makeatletter
\renewcommand{\mlsys@appearing}{\textit{Preprint.}}
\makeatother

\newtheorem{proposition}{Proposition}

\newcommand{\method}{TreeSpark}
\newcommand{\tautok}{\ensuremath{\tau}}

\mlsystitlerunning{TreeSpark}

\begin{document}

\twocolumn[
\mlsystitle{\method: Calibrated, Load-Adaptive Draft Trees for\\Semi-Autoregressive Speculative Decoding}

\begin{mlsysauthorlist}
\mlsysauthor{Huapeng Zhou}{boson}
\mlsysauthor{Huayu Wang}{uw}
\mlsysauthor{Xinyu Wang}{boson,mcgill}
\end{mlsysauthorlist}

\mlsysaffiliation{boson}{Boson AI}
\mlsysaffiliation{uw}{University of Washington}
\mlsysaffiliation{mcgill}{McGill University}
\mlsyscorrespondingauthor{Huapeng Zhou}{zhouhp.me@gmail.com}

\mlsyskeywords{speculative decoding, draft trees, LLM inference, serving}

\vskip 0.3in

\begin{abstract}
Speculative decoding accelerates language-model inference by letting a
cheap drafter propose tokens that the target model verifies in parallel.
Recent block drafters make drafting nearly free: a single backbone pass
emits an entire block of draft tokens. Draft trees promise a further
gain --- several alternative continuations verified in one target
forward --- but existing constructions rank candidates by per-position
marginals that ignore which parent a candidate extends, so on
semi-autoregressive drafters wider trees mostly add mis-ranked nodes; and
a tree of fixed size ignores how much speculation each decoding round,
and each serving load, can support. We introduce \method{}, which reads a
parent-conditioned distribution from the drafter's existing Markov head
at negligible cost, calibrates it into an edge-acceptance estimate, and
lets path survival govern everything else: best-first expansion,
per-round stopping, and a load-adaptive serving policy. Sampling siblings
without replacement, with matching residuals in recursive rejection,
keeps decoding lossless at any temperature. Adaptive trees improve on
matched fixed budgets at every temperature; against a tuned chain on the
same drafter, \method{} accepts $15$--$25\%$ more draft tokens per round
and decodes $8$--$14\%$ faster in single-request wall-clock, and under
rising load it gracefully shrinks the tree back to the chain. Code and
artifacts: \url{https://github.com/PopSoda2002/TreeSpark}.
\end{abstract}
]

\printAffiliationsAndNotice{}
\chead{\small\bf \method: Calibrated, Load-Adaptive Draft Trees}

\section{Introduction}
\label{sec:intro}

\begin{figure*}[t]
\centering
\includegraphics[width=0.92\textwidth]{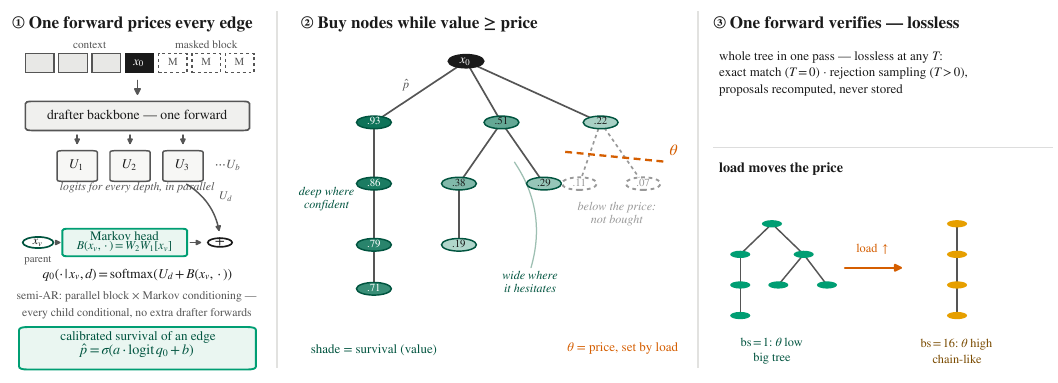}
\caption{One round of \method{}. \textbf{(1)}~A shared drafter backbone
produces block logits, while the rank-$r$ Markov head supplies a
parent-specific conditional distribution. \textbf{(2)}~Products of
calibrated edge-acceptance probabilities estimate path survival, and
best-first expansion stops when the best remaining candidate falls below
$\theta$. \textbf{(3)}~A single target forward verifies the resulting tree
by exact matching at $T{=}0$ or recursive rejection sampling at $T{>}0$.
The runtime adjusts $\theta$ as load changes, shrinking the tree toward the
chain when verification becomes more expensive.}
\label{fig:overview}
\end{figure*}

Speculative decoding reduces the number of target-model forward passes by
drafting several tokens and verifying them together
\citep{leviathan2023fast,chen2023accelerating}. Its performance depends on
both the number of tokens committed in each round and the cost of drafting
and verifying them. These quantities need not move together: acceptance
varies across token positions, requests, and datasets, while target
verification often dominates runtime \citep{liu2026illusion}. Verifying more
candidates may therefore increase accepted length yet reduce end-to-end
throughput.

A conventional draft is a chain: each token depends on its predecessor, but
one rejection invalidates the remaining suffix. A draft tree instead covers
multiple continuations around uncertain prefixes and verifies them in a
single target forward pass using a tree attention mask
\citep{miao2024specinfer,cai2024medusa}. That broader coverage requires more
verified tokens, and its benefit diminishes as batching makes verification
compute-bound \citep{liu2026illusion}. The right tree size consequently
depends on both the current decoding round and the current system load.

Parallel block drafters introduce a second challenge. Their backbone emits
one marginal distribution per depth, independent of the token selected at
the preceding depth. Building a tree directly from these marginals assigns
the same child distribution to different parents, so additional nodes need
not represent plausible continuations of the branches on which they appear.
An effective one-pass tree therefore needs both parent-specific conditional
distributions and a rule for deciding how many candidates to verify.

\method{}, illustrated in Figure~\ref{fig:overview}, obtains the
parent-specific distributions from the semi-autoregressive factorization of
DSpark \citep{dspark2026} and builds its per-round budget rule from calibrated
acceptance estimates. One backbone pass produces base logits $U_d$ for every
draft position. For a candidate at depth $d$ with parent token $x_p$,
DSpark's existing Markov head adds a low-rank transition bias,
\begin{equation}
q_{0,d}(\cdot\mid x_p)=\mathrm{softmax}(U_d+B(x_p,\cdot)).
\end{equation}
Evaluating another parent requires only an embedding lookup and a rank-$r$
matrix--vector product, rather than another backbone pass. The shared base
logits can thus support a tree while retaining the predecessor conditioning
used by the original chain. Without this correction, a marginal tree built
on the co-trained DSpark backbone underperforms the chain even with
$8\times$ its verification budget (\S\ref{sec:eval-main}).

Conditional probability alone is not a calibrated node value. Although raw
$q_0$ ranks edges well, it overestimates how often they match the target.
DSpark's confidence head does not resolve this problem because it never
observes the candidate token. A target verification pass, however, labels
every materialized edge. We use these labels to fit
$\hat p=\sigma(a\,\mathrm{logit}(q_0)+b)$ on edges whose ancestors were
accepted---the population relevant to path survival. The Qwen3-4B fit attains
held-out ECE $0.0105$. Each drafter--target pair requires only two fitted parameters,
which are then reused across domains and sampling temperatures. \method{}
prioritizes candidates by the product of $\hat p$ along each path and stops
when the best remaining product falls below $\theta$. The verification budget
is therefore chosen for each round, while the serving runtime selects
$\theta$ from a measured, load-dependent ladder
(\S\ref{sec:value}--\S\ref{sec:serving}).

Nonzero-temperature decoding also requires a correct sampling procedure.
Deterministic top-$k$ children are not proposal samples and therefore bias a
recursive rejection walk. \method{} instead draws siblings without
replacement from the residual $q_T$. During verification, it reconstructs
each proposal from the shared logits and sibling draw order before applying
recursive rejection sampling. This procedure preserves the target
distribution for any tree produced by the stopping rule. In a high-entropy
test, deterministic top-$k$ verification has total variation $0.46$, whereas
\method{} has empirical deviation comparable to the finite-sample reference
from direct target sampling (\S\ref{sec:sampling}).

Experiments on Qwen3-4B/8B/14B and six benchmarks show that parent
conditioning yields $9$--$16\%$ more accepted tokens per target forward
than DDTree at the same budget. At $T\in\{0,0.5,1\}$, adaptive stopping
lies above the matched fixed-budget acceptance--verification-cost frontier.
On Qwen3-4B, the final implementation improves accepted length by
$15$--$25\%$ and runs $8$--$14\%$ faster than the DSpark chain in
single-request wall-clock measurements. Serving measurements further show
why budget must remain a runtime decision: a large fixed tree that performs
well at batch size one becomes net-negative under load, whereas the
controller reduces tree size and eventually selects the chain.

\textbf{Contributions.}
\begin{itemize}
\item We construct a one-pass conditional tree by evaluating DSpark's Markov
  head for each concrete parent, without retraining or additional backbone
  passes (\S\ref{sec:method}).
\item We calibrate edge-acceptance probabilities and use path survival to set
  the verification budget per round, exposing $\theta$ as a load-dependent
  runtime control (\S\ref{sec:value}, \S\ref{sec:serving}).
\item We combine without-replacement expansion with recursive rejection
  sampling to preserve the target distribution exactly at nonzero
  temperature (\S\ref{sec:sampling}).
\item We evaluate three target-model scales on six to nine tasks under greedy
  and sampled decoding, with single-request and batched execution and
  ablations of conditioning, calibration, and fixed budgets
  (\S\ref{sec:eval}).
\end{itemize}

\section{Background and Related Work}
\label{sec:related}

\paragraph{Chain speculation and tree verification.}
Classical speculative decoding drafts one continuation and accepts a prefix
with a rejection sampler that preserves the target distribution
\citep{leviathan2023fast,chen2023accelerating}. SpecInfer represents several
continuations as a prefix-sharing token tree and verifies them in one target
forward pass \citep{miao2024specinfer}, while Medusa predicts multiple
children with parallel heads \citep{cai2024medusa}. Subsequent work improves
tree shape: EAGLE-2 grows dynamic trees from draft confidence
\citep{li2024eagle2}, and OPT-Tree and Sequoia optimize topology under
acceptance and hardware models \citep{wang2024opttree,chen2024sequoia}. In a
serving system, however, the additional verification work must also be
accounted for. A production-scale study finds that target verification
dominates runtime and that the small-batch advantage of wide trees can
disappear as batch size increases \citep{liu2026illusion}. \method{} therefore
treats tree size as a runtime decision rather than fixing one topology for
all rounds and loads.

\paragraph{One-pass and semi-autoregressive drafters.}
A drafter must balance proposal quality against its own latency. PRISM
refactors parameter use in autoregressive draft models to improve this
trade-off \citep{wang2026prism}. SpecDiff-2 uses discrete diffusion for
parallel drafting and trains the drafter to agree more closely with the
verifier \citep{sandler2026specdiff2}. DFlash produces a block in one
lightweight diffusion pass, with a marginal distribution at each position
\citep{chen2026dflash}. Domino restores causal dependence with a GRU
correction \citep{domino2026}, whereas DSpark uses a low-rank Markov
correction based on the immediately preceding token \citep{dspark2026}.
\method{} reuses DSpark's released factorization: it shares the parallel
backbone logits across the tree and evaluates the Markov correction for each
parent. It neither retrains nor realigns the drafter; calibration only
estimates whether the target will accept a particular edge proposed by the
fixed drafter.

\paragraph{Trees from parallel drafts.}
DDTree grows a best-first, fixed-budget tree directly from DFlash's
per-position marginals \citep{ddtree2026}. TAPS adds target-aware path
selection \citep{taps2026}; CaDDTree introduces a cost-aware stopping rule
under a verification-latency model \citep{caddtree2026}; and TreeFlash learns
an auxiliary approximation to autoregressive conditioning
\citep{treeflash2026}. SpecBlock adds iterative block drafting
\citep{specblock2026}, while Spec-AUF changes masked-drafter training to match
accept-until-fail inference \citep{specauf2026}. These methods modify
candidate selection, stopping, drafting, or training for a parallel drafter,
but a marginal backbone still
assigns the same distribution to distinct parents at a given depth. In our
ablation, applying such a marginal construction to DSpark underperforms its
original chain. Evaluating the available Markov correction separately for
each parent eliminates this failure mode.

\paragraph{Concurrent parent-conditioned parallel trees.}
PCTree is architecturally closest to our conditional expansion: it also
reuses DSpark's Markov head to score concrete parents and batches frontier
expansions. It operates under a fixed node budget and reports greedy
verification \citep{pctree2026}. DominoTree makes the analogous change to
Domino by using the GRU-corrected conditional, rather than a marginal score,
in its heap \citep{dominotree2026}. Its sampled decoder uses deterministic
drafts with exact-match acceptance, and its adaptive experiment uses
uncalibrated path probabilities. JetSpec instead trains a causal parallel
head for conditioned tree drafting \citep{jetspec2026}. Together, these
concurrent works show that one-pass trees benefit from parent-conditioned
scores. \method{}
focuses on the complementary questions of budget control and sampling
correctness: it calibrates acceptance on the survival-conditioned population,
stops using the resulting absolute path value, and applies recursive
rejection sampling to genuine proposals at nonzero temperature.

\paragraph{Speculation under serving load.}
Speculation can lose goodput when concurrency makes verification
compute-bound \citep{liu2024smartspec,liu2026illusion}. TurboSpec and
AdaSpec adapt chain length to load or latency targets
\citep{liu2024smartspec,adaspec2025}; DSpark likewise schedules chain length
from confidence and engine cost. SparseSpec co-designs self-speculation with
batching, delayed verification, and KV management \citep{zhao2026specgen},
while CaDDTree chooses the size of a marginal tree from a latency model.
\method{} estimates the marginal value of each additional conditional-tree
node and exposes the stopping threshold to the runtime. Our measured engine
uses a threshold ladder indexed by $(\text{batch size},T)$ and falls back to
the chain when no branching configuration improves wall-clock performance
(\S\ref{sec:serving}).

\section{\method{}: Markov-Conditioned Draft Trees}
\label{sec:method}

\subsection{Preliminaries: the semi-autoregressive drafter}
\label{sec:prelim}

A DSpark drafter attached to a target LLM has three components: (i) a small
bidirectional backbone over a block of $b$ positions, conditioned on features
from $t$ target layers; (ii) a rank-$r$ Markov head
$B(x, v) = \langle W_1[x], W_2[v]\rangle$, where $r \ll |V|$; and (iii) a
confidence head $c(h, x) = \sigma(w^\top [h; W_1[x]])$. A single backbone
pass over $[x_{\text{anchor}}, \texttt{[M]}, \dots, \texttt{[M]}]$ produces
hidden states $h_1, \dots, h_b$ and base logits $U_1, \dots, U_b$.

\textbf{Block semantics.} We use one-based indices in the paper to match
Figure~\ref{fig:overview}; the implementation tensors are zero-based. In a
checkpoint tensor, row $m{-}1$ predicts the draft token at offset $m$ after
the anchor, so each block produces $b$ drafts. This LM-shifted alignment must
be made explicit because an
in-place interpretation can appear to work on DSpark: the Markov head alone
acts as a bigram model and can conceal the offset error in end-to-end tests.
Appendix~\ref{app:probe} reports an alignment probe that distinguishes the
two interpretations.

The chain decoder samples token $k$ from
$\mathrm{softmax}(U_k + B(x_{k-1}))$. The backbone provides parallel
context features, and the Markov head contributes first-order sequential
dependence.

\subsection{Path-conditioned best-first expansion}
\label{sec:expansion}

Each node in a draft tree has a depth $d \in \{1,\dots,b\}$ relative to the
anchor and a parent token $x_p$. From a single backbone pass, \method{}
constructs this tree in best-first order, subject to a node cap $N$:

\begin{enumerate}
\item \textbf{Child distributions.} A candidate token at depth $d$, whose
  parent token is $x_p$, is drawn from
  $q_{0,d}(\cdot \mid x_p) =
  \mathrm{softmax}(U_d + B(x_p, \cdot))$. Evaluating
  $B(x_p,\cdot)$ requires one row lookup from $W_1$ and one
  $|V| \times r$ matrix--vector product, but no additional drafter pass.
  Candidates on distinct branches at the same depth therefore receive distinct,
  parent-conditioned distributions.
\item \textbf{Priorities.} A max-heap holds candidate children keyed by
  \emph{calibrated path survival}, the product of
  $\hat{p}(\text{edge})$ (Eq.~\ref{eq:phat}, \S\ref{sec:value}) along the
  path from the root. This product estimates the probability that the entire
  path will be accepted and hence the node's marginal contribution to
  expected accepted length. The heap selects the globally best candidate.
  At $T{=}0$, the candidate is the highest-probability unused token under
  its parent's conditional; at $T{>}0$, it is drawn without replacement
  (\S\ref{sec:sampling}). Materializing a node creates one slot for its first
  child and, up to a cap of $K$ siblings, one slot for its parent's next
  child. Expansion continues until the stopping rule in \S\ref{sec:adapt}
  fires; a fixed budget $N$ is the special case $\theta{=}0$. Raw path
  log-probability $\sum \log q$ and the trained confidence head are evaluated
  only as ablations (\S\ref{sec:eval-main}, \S\ref{sec:value}).
\item \textbf{Verification.} The flattened tree is verified in one target
  forward pass using a standard tree attention mask and per-depth position
  ids. The decoder commits the longest surviving root path---using exact
  target-token matching for greedy decoding or the procedure in
  \S\ref{sec:sampling} for
  sampling---followed by one correction or bonus token from the target
  posterior. As in chain DSpark, the hidden states on the accepted path become
  the target features for the next round.
\end{enumerate}

Algorithm~\ref{alg:expand} summarizes the procedure. Greedy and sampled
decoding differ only in how a child token is selected (line~10 versus
line~12). At $T{>}0$, the heap uses the admissible pre-draw bound
$S(v)\,\hat p(\max q_0\text{ unused})$, and queued parents are processed in
one batched GPU operation. Because a sampled value is never used to decide
whether its own slot is admitted, batching does not affect the losslessness
argument in \S\ref{sec:sampling}.

\begin{algorithm}[t]
\caption{\method{} tree expansion (one round)}
\label{alg:expand}
\begin{algorithmic}[1]
\REQUIRE anchor $x_0$, target features $F$, block length $b_{\rm blk}$,
price $\theta$, temperature $T$, caps $(N, K)$; Markov head
$B(\cdot,\cdot)$, calibration $(a,b)$
\STATE $U_{1:b_{\rm blk}} \gets
  \mathrm{Backbone}(F, [x_0, \texttt{[M]}, \dots, \texttt{[M]}])$
  \hfill $\triangleright$ one drafter forward
\STATE set the root depth to $d{=}0$; for any parent $v$ at depth $d$, define
  $q_{0,d+1}(\cdot\,|\,v) \gets
  \mathrm{softmax}(U_{d+1} + B(x_v,\cdot))$;\;
  if $T{>}0$, define $q_{T,d+1}$ likewise from logits$/T$;\;
  $\hat p(\cdot\,|\,v) \gets \sigma(a\,\mathrm{logit}\,q_{0,d+1} + b)$
  \hfill $\triangleright$ score with $q_0$, draw with $q_T$
\STATE $S(\mathrm{root}) \gets 1$;\;
  $H \gets \{\mathrm{slot}(\mathrm{root})\}$;\; $\mathcal{T} \gets \emptyset$
\WHILE{$|\mathcal{T}| < N$ \textbf{and} $H \neq \emptyset$}
  \STATE pop from $H$ the slot of parent $v$ maximizing the candidate
    survival $S(v)\cdot\hat p(\text{best unused token}\,|\,v)$
  \IF{that survival $< \theta$}
    \STATE \textbf{break}
    \hfill $\triangleright$ marginal node below its verification price
  \ENDIF
  \IF{$T = 0$}
    \STATE $x_c \gets$ best unused token of $q_{0,d+1}(\cdot\,|\,v)$
  \ELSE
    \STATE $x_c \sim q_{T,d+1}(\cdot\,|\,v)$ restricted to unused tokens
    \hfill $\triangleright$ without replacement
  \ENDIF
  \STATE $S(c) \gets S(v)\cdot\hat p(x_c\,|\,v)$;\;
    add $c = (x_c, d{+}1)$ to $\mathcal{T}$ with parent $v$
  \STATE push $\mathrm{slot}(c)$ if $d{+}1 < b_{\rm blk}$;\;
    re-push $\mathrm{slot}(v)$ if $v$ has $< K$ children
\ENDWHILE
\STATE \textbf{return} $\mathcal{T}$ as (tokens, parents, depths);
  verification recomputes each visited parent's conditionals from
  $(U,\,\text{draw order})$, using exact matching at $T{=}0$ and recursive
  rejection sampling at $T{>}0$
\end{algorithmic}
\end{algorithm}

Our primary ablation is a \emph{marginal} tree that scores children using
$\mathrm{softmax}(U_d)$ alone, transplanting DDTree's construction onto
the DSpark drafter. Because the backbone was co-trained with the Markov head,
these uncorrected marginals misrank branches and spend verification budget on
unlikely paths (\S\ref{sec:eval-main}).

\subsection{Lossless sampled trees}
\label{sec:sampling}

Greedy tree verification reproduces greedy target decoding by construction.
Sampling requires more care because an incorrect proposal procedure can
silently change the output distribution.

\textbf{The failure mode.} A tempting approach is to expand the tree
deterministically with the top-$k$ children and then apply the standard
recursive rejection walk: accept each child $x$ with probability
$\min(1,p(x)/q(x))$ and subtract $q$ from the target residual after a
rejection \citep{miao2024specinfer}. This procedure is biased because
recursive rejection requires candidates sampled from the proposal, not
deterministically selected modes. Table~\ref{tab:lossless} shows that the
bias can be large. When $q \approx p$, the first top-$k$ candidate is almost
always accepted, concentrating the output near the argmax.

\textbf{Without-replacement sampled expansion.} For children at depth $d$,
\method{} draws the $j$th sibling from the \emph{residual proposal}
\begin{equation}
q_{T,d}^{(j)} \;=\; \mathrm{norm}\!\big(q_{T,d} \cdot \mathbb{1}[\,\cdot \notin
\{x^{(1)},\dots,x^{(j-1)}\}\,]\big).
\end{equation}
Here, previously drawn siblings are removed from the node's conditional
$q_{T,d}$, and the remaining mass is renormalized. Verification reconstructs
the same $q_{T,d}^{(j)}$ bit-for-bit from the backbone output and draw order rather
than storing it (\S\ref{sec:eval-temp2}). It visits siblings in draw order,
accepts $x^{(j)}$ with probability
$\min(1,p_j(x^{(j)})/q_{T,d}^{(j)}(x^{(j)}))$, and after rejection updates the
target residual to $p_{j+1}=\mathrm{norm}(\max(p_j-q_{T,d}^{(j)},0))$.

\begin{proposition}
\label{prop:lossless}
For any node, any number of siblings, and any adaptive policy whose choice of
the next node to expand depends only on already-materialized values, the tree
walk above emits tokens drawn exactly from the target distribution $p$, provided
that each new slot is admitted before its token value is observed and that
draws are neither discarded nor reordered according to their values.
\end{proposition}

\emph{Proof sketch.} At stage $j$, conditioned on all preceding siblings
being rejected, the residual target $p_j$ is fixed and
$x^{(j)}\sim q_{T,d}^{(j)}$. This is an ordinary one-step speculative sampling
operation and therefore preserves $p_j$ by the standard argument
\citep{leviathan2023fast}. Composing stages over siblings and recursing over
depth preserves $p$. The tree shape may depend on values already drawn,
because this only chooses which valid proposal sequence to extend. In
contrast, discarding or reordering a draw according to its value would
violate the requirement that the stage-$j$ candidate follow
$q_{T,d}^{(j)}$.
\hfill$\square$

Without-replacement draws also improve efficiency: siblings cannot duplicate
one another, which redistributes the depth--sibling budget as candidates are
removed. In our
initial implementation, i.i.d.\ draws with replacement caused sampled trees
to underperform the chain (\tautok{} $2.10$ versus $4.94$ at equal budget in
an isolated H200 rerun over 150 prompts and three repetitions).
Without-replacement expansion removes this collapse
(\S\ref{sec:eval-temp}).

\subsection{What predicts edge acceptance}
\label{sec:value}

Expansion order, branching, and total tree size should all reflect the
probability that the target will accept an edge. Tree verification supplies
the required labels without an extra target pass. Under a tree attention
mask, the row for node $i$ attends only to its root path, so the verifier
reveals whether every materialized node---accepted or not---matches the
target model's greedy prediction at its parent. We instrument greedy
Markov-tree decoding
with $N{=}28$ and $K{=}8$ on 8 prompts from each of six benchmarks, yielding
34{,}216 labeled edges for Qwen3-4B. We use global even-indexed prompts for
fitting and odd-indexed prompts for held-out evaluation. After selecting the
fitting population with this split, we refit the deployed two-scalar map on
all 12{,}561 ancestors-accepted edges. Its coefficients are $0.67/{-}0.46$;
the split-only refit gives similar values, $0.69/{-}0.44$.
Table~\ref{tab:value} states explicitly which rows are pooled descriptions
and which use the held-out split.

The drafter provides two possible edge-acceptance predictors: its trained
confidence head $c$ and the Markov probability $q(x\mid x_p)$ assigned to
the candidate token. Table~\ref{tab:value} clearly favors the latter. The
confidence head ranks edges poorly (AUC $0.699$); adding it to a logistic
model on top of $\mathrm{logit}(q)$ changes AUC by only $0.0004$, with a
fitted weight of $0.01$. This result follows from the architecture: the
confidence head receives $(h_d,W_1[x_p])$ but not the candidate token,
so it assigns the same score to all siblings. As a path prior in the earlier
fixed-budget sweeps, it adds only ${+}0.031$--$0.064$ accepted tokens per
target forward (\tautok{}) under greedy decoding in the isolated
stored-proposal H200 rerun and reduces the end-to-end acceleration ratio
(AR; matched target-only latency divided by speculative latency) for
$N\in\{14,28,56\}$ ($2.871\times\to2.573\times$ at $N{=}56$). At
$T{=}1$ and $N{=}56$, both metrics decline: \tautok{}
$5.393\to5.292$ and AR speedup $1.918\times\to1.453\times$ (225 prompts,
three H200 repetitions). The AR comparisons in this paragraph are internal
to that stored-proposal harness and archived with the artifacts.

The Markov probability ranks edges well (pooled AUC $0.883$) but
overestimates acceptance: its mean prediction is $0.622$, compared with a
realized rate of $0.498$. A two-parameter Platt map corrects the scale, but
the fitting population matters. Path-survival products condition an edge on
all of its ancestors being accepted. Fitting on every materialized edge also
includes counterfactual contexts whose ancestors were rejected, which biases
the estimates for high-confidence edges downward. On held-out
ancestors-accepted edges, this all-edges fit has ECE $0.059$. Fitting only
on ancestors-accepted training edges reduces held-out ECE to $0.0105$. We
then refit that selected model on the pooled ancestors-accepted set for
deployment, obtaining
\begin{equation}
\hat{p}(\text{edge}) \;=\;
\sigma\!\big(a \cdot \mathrm{logit}\,q_0 + b\big), \qquad
a = 0.67,\; b = -0.46.
\label{eq:phat}
\end{equation}
The held-out split gives ECE $0.0105$ and AUC $0.948$. Adding sibling rank
provides no further improvement. The resulting value function consists of
two scalars fitted once from labels already produced during decoding.

\textbf{Generality of the two scalars.} A single fit transfers across content
domains. ECE ranges from $0.019$ to $0.035$ on the six evaluation suites;
per-domain refits improve it by at most $0.01$, and leave-one-domain-out fits
keep held-out ECE at or below $0.040$ with nearly unchanged coefficients
($a\in[0.66,0.71]$). The model pair has a larger effect: scale-specific
refits yield $a=0.69/0.72/0.79$ for Qwen3-4B/8B/14B, with test ECE
$0.008$--$0.015$. Calibration is therefore specific to a drafter--target
pair rather than a domain. No temperature-specific refit is needed. Sampled
expansion draws from the temperature-scaled conditional $q_T$, which must be
reconstructed for lossless verification (\S\ref{sec:sampling}), while the
calibrator receives the corresponding unscaled $q_0$, the input on which it
was fitted. Both are recovered analytically from the same logits. Feeding
$q_T$ to the calibrator instead creates an input mismatch and removes the
observed robustness across temperatures (\S\ref{sec:eval-temp2}).

\subsection{Calibrated adaptive budgets}
\label{sec:adapt}

For a calibrated $\hat p$, a candidate node's expected contribution to
accepted length is its \emph{path survival},
$\prod_{e\in\text{path}}\hat p(e)$. This quantity provides both a priority
for best-first expansion and a stopping signal that a fixed node budget
cannot express. In the adaptive variant of \method{}, cumulative
$\log\hat p$ determines heap order, and expansion stops when the best
remaining path survival falls below $\theta$, subject to a node cap of 64.
The threshold has a direct interpretation as the minimum expected accepted
tokens per additional verified token. A serving scheduler can therefore set
this marginal price from the current load (\S\ref{sec:serving}). Because
$\hat p$ and $\theta$ affect only tree shape, they do not change the
losslessness result in Proposition~\ref{prop:lossless}. Easy rounds receive
larger trees, while difficult rounds approach a short chain; the budget is
determined per round rather than by a fixed offline choice.

\section{Load-Aware Tree Shaping for Serving}
\label{sec:serving}

Better offline acceptance does not necessarily translate into higher serving
goodput. At batch size 1, verifying $56$ tokens costs only slightly more than
verifying $7$. At higher concurrency, however, verification FLOPs compete
across requests, and oversized trees can reduce aggregate goodput
\citep{liu2024smartspec}. The tree budget needed to outperform the chain also
increases with temperature (\S\ref{sec:eval-temp}). Because all tree shapes
share one drafter-backbone pass, \method{} can vary the shape from a
$56$-node tree to a $7$-node chain, or disable speculation, without changing
the drafter. We extend DSpark's load-aware control of draft length to
load-aware control of tree shape.

\begin{algorithm}[t]
\caption{Price-ladder construction (offline, once per deployment)}
\label{alg:ladder}
\begin{algorithmic}[1]
\REQUIRE engine, measured step-latency model $L(\text{bs}, N)$, operating grid
$\mathcal{B} \times \mathcal{T}$, chain baseline
\FOR{$(\text{bs}, T) \in \mathcal{B} \times \mathcal{T}$}
  \STATE $\Theta \gets$ neighborhood of the cost-model optimum
    $\arg\max_{\theta}\; \text{bs} \cdot
    \tautok{}(\theta, T) \,/\, L(\text{bs}, \bar N(\theta, T))$
    \hfill $\triangleright$ cost model proposes candidate prices
  \FOR{$\theta \in \Theta$}
    \STATE $\mathrm{margin}(\theta) \gets$ same-GPU interleaved A/B
      vs.\ the chain: engines prebuilt, arms alternating rep by rep,
      warm-up rep discarded, paired diffs
      \hfill $\triangleright$ measurement selects the deployed price
  \ENDFOR
  \IF{$\max_\theta \mathrm{margin}(\theta) > 0$ with sign-consistent
    wins}
    \STATE $\theta^{*}(\text{bs}, T) \gets \arg\max_\theta
      \mathrm{margin}(\theta)$
  \ELSE
    \STATE $\theta^{*}(\text{bs}, T) \gets \bot_{\rm chain}$
      \hfill $\triangleright$ sentinel selecting $K{=}1$, $N{=}b$
  \ENDIF
\ENDFOR
\STATE \textbf{return} the policy ladder; at serving time $\theta$ (or the
  chain sentinel) is a lookup per round, re-derived only when hardware,
  drafter, or temperature
  regime changes
\end{algorithmic}
\end{algorithm}

The controller has three components:

\begin{enumerate}
\item \textbf{Cost model.} We express measured step latency as
  $L(\text{bs},N)$, where $N$ is the number of verified tree nodes and
  $T$ remains reserved for sampling temperature.
\item \textbf{Per-request budget allocation.} In each round, the scheduler
  chooses $\{N_i\}$ to maximize expected accepted tokens per unit step time.
  Calibrated path survival (Eq.~\ref{eq:phat}) gives the expected contribution
  of each candidate. Greedy water-filling suffices for the concave
  per-request gain curves; under a uniform marginal price, this allocation
  reduces exactly to the $\theta$ stopping rule. Thus, the scheduler controls
  tree size by setting $\theta$ from the current load.
\item \textbf{Graceful degradation.} As load increases, the optimizer
  contracts trees toward chains and then toward no speculation, recovering
  DSpark's trimmed-chain scheduler as a one-dimensional special case.
  Proposition~\ref{prop:lossless} applies to every such shaping policy. In
  the measured engine at $\text{bs}{=}16$, the best branching trees trail the
  chain by $2.4\%$ at $T{=}1$ and $1.5\%$ at $T{=}0.5$. The policy therefore
  selects the code-identical chain member, giving parity in those cells.
  Algorithm~\ref{alg:ladder} specifies the construction,
  Table~\ref{tab:ladder} reports the measured ladder, and
  Figure~\ref{fig:adapt} shows how the selected price and per-request tree
  sizes change with the active batch.
\end{enumerate}

An independent final-recompute H200 reproduction of the
Table~\ref{tab:temp} protocol, archived with the artifacts, supports both the
result and the need for
deployment-specific calibration. The adaptive tree retains
${+}5.7$--$12.8\%$ throughput over the chain at every temperature, while
the hardware changes the $T{=}0$ wall-clock optimum from $\theta{=}0.02$ to
$\theta{=}0.05$. Algorithm~\ref{alg:ladder} therefore rebuilds the ladder
for each deployment rather than fixing it globally.

\begin{table}[t]
\caption{Measured price ladder $\theta^{*}(\text{bs},T)$ constructed by
Algorithm~\ref{alg:ladder}. Each strict-tree entry reports the selected
$\theta$, its wall-clock margin over the chain, and its paired win count
(A100 engine; same-GPU interleaved A/B runs with prebuilt engines and the
warm-up repetition discarded). Each selected $\theta$ also beat the
neighboring probes. At $\text{bs}{=}16$ and $T{>}0$, no branching tree was
faster than the chain. At $T{=}1$, the probes
$\theta\in\{0.05,0.10,0.15,0.4\}$ trailed by $12/4.6/2.4/7\%$; at
$T{=}0.5$, $\theta{=}0.15$ trailed by $1.5\%$ (0/5). The ladder therefore
selects its chain member in those cells, giving parity by code identity.
Intermediate batch sizes use the shipped interpolation stated below;
Fig.~\ref{fig:serving}(a) provides modeled context rather than a
cell-by-cell certification.}
\label{tab:ladder}
\vskip 0.1in
\centering
\scriptsize
\setlength{\tabcolsep}{2.5pt}
\resizebox{\columnwidth}{!}{%
\begin{tabular}{lccc}
\toprule
& $T{=}0$ & $T{=}0.5$ & $T{=}1.0$ \\
\midrule
$\text{bs}{=}1$ & $0.02$: ${+}7.5\%$ (5/5) & $0.05$: ${+}6.1\%$ (3/3) &
  $0.05$: ${+}3.0$--$3.4\%$ (7/8) \\
$\text{bs}{=}16$ & $0.15$: ${+}2.1\%$ (4/4) & chain (parity) & chain (parity) \\
\bottomrule
\end{tabular}
}
\vskip -0.1in
\end{table}

\begin{figure*}[t]
\centering
\includegraphics[width=0.9\textwidth]{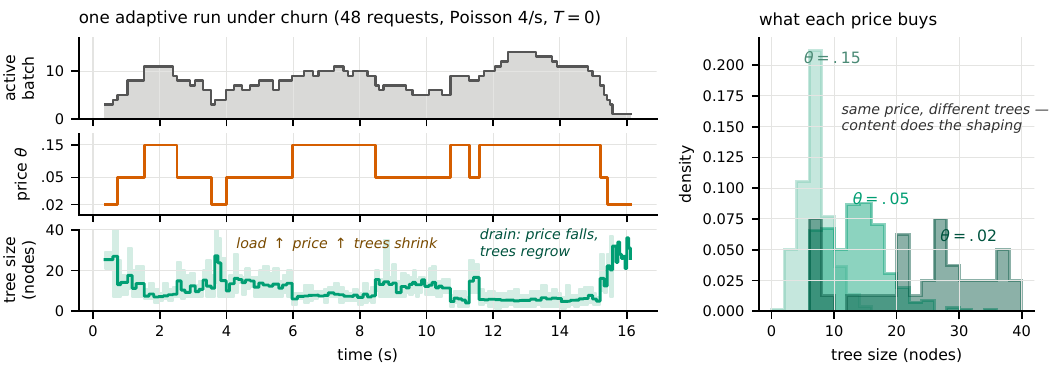}
\caption{Round-by-round $\theta$ adaptation under churn (48 requests,
Poisson arrivals at $4$/s, $T{=}0$, A100). \textbf{Left:} The ladder reads
the active batch each round. Trees shrink as the batch fills and grow again
as it drains; near the end of the run, the controller returns to the
${\sim}25$-node regime associated with $\theta{=}0.02$.
\textbf{Right:} Distribution of tree sizes at each selected price. Higher
$\theta$ shifts the distribution toward the chain, while variation within
each price reflects the survival profile of the individual request.}
\label{fig:adapt}
\end{figure*}

\begin{figure*}[t]
\centering
\includegraphics[width=0.62\textwidth]{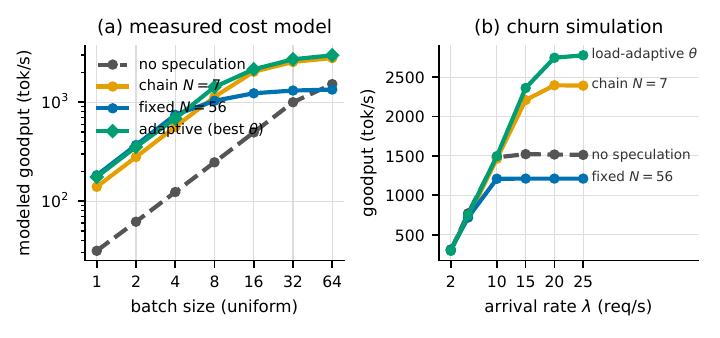}
\caption{Serving results at $T{=}0$; colors denote the same policy in both
panels. \textbf{(a)}~Combining measured step costs with measured acceptance shifts
the best shape from the largest tree in the memory-bound regime
($\text{bs}{\le}4$) to small adaptive trees in the compute-bound regime
($\text{bs}{\ge}16$); the adaptive curve is the modeled best $\theta$ at
each load. \textbf{(b)}~In a scheduler-level simulation with
request churn, each fixed shape saturates at a different rate. The
load-adaptive tree reaches the highest capacity, about ${+}15\%$ over the chain;
the fixed $N{=}56$ tree, which is optimal at $\text{bs}{=}1$, has the
lowest capacity under load. Real-engine measurements in the text confirm
the same trends.}
\label{fig:serving}
\end{figure*}

\noindent\textbf{Measured cost model.} At $T{=}0$, we directly measure decode-step
latency $L(\text{bs},k)$ on an A100-40GB with Qwen3-4B in bf16 and
512-token cached prefixes, using the median of 15 measurements. We combine
these latencies with the per-request $(\tautok{},\text{verified})$
operating points from \S\ref{sec:eval-adapt} to estimate goodput as
$\text{bs}\cdot\tautok{}/L$ (Fig.~\ref{fig:serving}(a); the full grid appears
in Table~\ref{tab:goodput}). Up to $\text{bs}{=}4$, latency remains
$34$--$38$\,ms for both $k{=}1$ and $k{=}56$, so the largest tree performs
best. From $\text{bs}{=}16$ onward, verification becomes compute-bound; at
$\text{bs}{=}64$, $k{=}56$ costs $7.3\times$ as much as $k{=}1$.
Consequently, the offline-optimal $N{=}56$ tree becomes net-negative,
whereas the modeled best-$\theta$ envelope follows the upper frontier. At
$\text{bs}{=}64$, it provides ${+}7\%$ goodput over the best fixed shape
of any size, and no single fixed shape remains within $20\%$ of the envelope
at every load.

\noindent\textbf{Scheduler-level simulation under churn.} We next simulate
continuous batching using measured costs. Requests arrive as a Poisson
process into a 64-slot batch; each step uses $L(\text{bs},\bar{k})$ from the
cost model. Fixed policies replay each prompt's measured rounds in order.
For the adaptive policy, each step selects $\theta$ from the current batch
size and samples a round from the same prompt's trace under that $\theta$,
preserving prompt-level difficulty after a policy switch. We use
$\theta{=}0.02$ for $\text{bs}{\le}4$, $0.05$ for
$5{\le}\text{bs}{\le}8$, and $0.15$ above 8, matching the $T{=}0$
dynamic-engine controller. Each fixed shape reaches a different saturation
point: $N{=}56$ at $1.2$k tok/s, $N{=}28$ at $1.8$k, and the chain at
$2.4$k. The load-adaptive policy reaches $2.77$k tok/s, ${+}15\%$ over the
chain and ${+}82\%$ over no speculation. At the two highest offered loads,
it exceeds the best fixed $\theta$ selected in hindsight by ${+}3.4$--$4.5\%$
and lowers P99 completion latency by $48\%$ relative to the chain at
$\lambda{=}20$\,req/s. The simulation
excludes drafter cost and prefill contention; their effect in a production
deployment remains to be measured.

\noindent\textbf{Real dynamic batching in a lossless research engine.}
Finally, we implement a minimal dynamic-batching engine. It uses a
preallocated per-row KV buffer, writes every tree token to a fresh offset,
and never moves tokens. Rejected nodes never enter the row's valid set, and
a per-step ancestor-bitmask attention mask defines visibility. Requests are
prefilled on admission, verified in one physically batched forward pass per
step, and removed when complete. At batch size 1, the engine reproduces the
reference implementation token-for-token.

On an A100 at $T{=}0$, a pool of 48 requests spends $14.6$\,s in cumulative
target-model forwards under target-only decoding and $4.3$\,s under adaptive
\method{}, a $3.4\times$ reduction. End to end, the chain improves throughput
over no speculation by ${+}16\%$ ($326$ versus $281$ tok/s), the fixed
$N{=}56$ tree is net-negative at $224$ tok/s, and the adaptive policy exceeds
the chain by the certified ${+}2.1\%$ in Table~\ref{tab:ladder} as $\theta$
contracts its trees toward chains.

At batch size 1, the same dynamic-engine harness gives the chain a
$4.43\times$ AR speedup over its matched target-only path, consistent with
DSpark's published serving factor. The adaptive tree adds ${+}11\%$
throughput ($119.6$ versus $108.1$ tok/s) and reduces cumulative time in
target-model forwards from $8.2$ to $6.9$\,s: per-round target cost is flat
in $k$ at this batch size, while higher \tautok{} requires fewer rounds. With
both configurations CUDA-graphed, the gain is ${+}7\%$.

Every wall-clock comparison in these dynamic-engine measurements uses
same-GPU interleaved A/B runs with prebuilt
engines, alternating repetitions, and the warm-up discarded. Cross-GPU
comparisons on the shared host showed $5$--$70\%$ tenant noise and are not
used. Cross-request batched drafting and production kernels are still needed
for deployment; the validated scheduler interface they require is $\theta$
as a per-round price.

\section{Offline Evaluation}
\label{sec:eval}

\subsection{Setup}
\label{sec:setup}

\textbf{Models.} We evaluate Qwen3-4B, Qwen3-8B, and Qwen3-14B as target
models. Each target uses the corresponding official DeepSpec checkpoints:
\texttt{dspark\_\allowbreak qwen3\_\allowbreak *\_\allowbreak block7}, a
semi-autoregressive drafter with
$b{=}7$ and $r{=}256$, and
\texttt{dflash\_\allowbreak qwen3\_\allowbreak *\_\allowbreak block7},
the same architecture with the
Markov and confidence heads removed. Both drafter families were trained with
the same pipeline and data, isolating the mechanism studied here.

\textbf{Benchmarks.} The primary evaluation uses 25 prompts from each of
GSM8K, MATH-500, HumanEval, MBPP, MT-Bench, and Alpaca, matching the
drafters' released evaluation suites. We generate 200 new tokens per prompt
with the chat template and thinking disabled, as in drafter training. We
report \tautok{}, the number of generated tokens per target forward, equal
to accepted draft length plus one. Every configuration is lossless; under
greedy decoding, all methods therefore produce identical outputs, leaving
\tautok{} as the quantity that varies.

\textbf{Supplemental protocol.} We additionally evaluate AIME25,
LiveCodeBench, and Arena-Hard-v2. The sampling evaluation uses 25 prompts per
task, while the adaptive evaluation uses 17 calibration-disjoint prompts per
task. Wall-clock results are means over three full H200 repetitions in bf16
with SDPA. Timing includes prompt processing through output sampling but
excludes model and dataset loading. We report \tautok{}, AR speedup over the
matched target-only decoder, and verified-node utilization, defined as
committed draft-path nodes divided by verified non-root nodes. For all H200
measurements, the AR denominator is our purpose-built, KV-cached target-only
decoder with the same model, prompts, sampling parameters, token limit, and
timing boundary.

\textbf{Methods.} We compare the DFlash chain; DDTree, a marginal best-first
tree on DFlash, with $N\in\{7,14,28,56\}$; the DSpark chain; and \method{}
at the same budgets. We also evaluate a marginal-scored tree on the DSpark
drafter as the ablation in finding (1). A chain is a linear tree with
$K{=}1$ and $N{=}7$, so verification-cost comparisons at equal $N$ are
exact.

\subsection{Main results: parent conditioning makes semi-AR trees effective}
\label{sec:eval-main}

\begin{table*}[t]
\caption{Accepted tokens per target forward (\tautok{}) under greedy
decoding on 150 prompts from the original six benchmarks. Budgets count
verified draft tokens per round ($N$ is a fixed per-round node quota; the
$\theta$ stop is not used here, so this isolates parent conditioning at
budgets matched to DDTree --- \S\ref{sec:eval-adapt} compares quota
against price). Chains cannot exceed $N{=}7$: the block-7 drafter yields
seven depths per backbone pass, and only branching spends more.
Markov-conditioned trees improve monotonically with budget and outperform
marginal DDTree at every matched budget. Finding (1) reports the
marginal-scored DSpark ablation, which underperforms the chain.}
\label{tab:main}
\vskip 0.1in
\centering
\small
\begin{tabular}{llcccc}
\toprule
Target & Method & $N{=}7$ & $N{=}14$ & $N{=}28$ & $N{=}56$ \\
\midrule
\multirow{4}{*}{Qwen3-4B}
& DFlash chain            & 3.86 & --   & --   & --   \\
& DDTree                  & 4.18 & 4.60 & 4.89 & 5.13 \\
& DSpark chain            & 4.39 & --   & --   & --   \\
& \method{}               & \textbf{4.57} & \textbf{5.18} & \textbf{5.50} & \textbf{5.81} \\
\midrule
\multirow{4}{*}{Qwen3-8B}
& DFlash chain            & 3.96 & --   & --   & --   \\
& DDTree                  & 4.25 & 4.66 & 4.98 & 5.19 \\
& DSpark chain            & 4.58 & --   & --   & --   \\
& \method{}               & \textbf{4.79} & \textbf{5.37} & \textbf{5.78} & \textbf{6.02} \\
\midrule
\multirow{4}{*}{Qwen3-14B}
& DFlash chain            & 4.13 & --   & --   & --   \\
& DDTree                  & 4.39 & 4.80 & 5.13 & 5.33 \\
& DSpark chain            & 4.61 & --   & --   & --   \\
& \method{}               & \textbf{4.83} & \textbf{5.45} & \textbf{5.88} & \textbf{6.11} \\
\bottomrule
\end{tabular}
\vskip -0.1in
\end{table*}

Table~\ref{tab:main} summarizes the greedy results. Three findings hold at
every model scale and on every benchmark.

\textbf{(1) Marginal trees fail on the evaluated DSpark drafter.} Applying
DDTree's construction to DSpark gives $2.60$--$2.72$ \tautok{} at $N{=}7$ and only
$3.14$--$3.33$ at $N{=}28$. Thus, even at $4\times$ the chain's budget, the
marginal tree remains below the $N{=}7$ chain; it also remains below at
$8\times$ budget (not shown). The co-trained backbone marginals do not rank
branches effectively, so most additional nodes are assigned to poor paths.
On DFlash, where marginals are the sole training signal, DDTree behaves as
expected in our runs: it improves over the chain by ${+}6$--$8\%$ at equal
budget, with larger gains at larger budgets. The failure is therefore
specific to the semi-autoregressive co-training setup, not to tree drafting.

\textbf{(2) Markov conditioning makes trees effective.} At equal
verification budget, \method{} outperforms the chain ($4.57$ versus $4.39$
on 4B), and its \tautok{} increases monotonically with budget. At $N{=}56$,
it improves over the chain by ${+}31$--$33\%$. Relative to DDTree at the
same budget, the gain is ${+}9$--$16\%$. Moreover, \method{} at $N{=}14$
matches or exceeds DDTree at $N{=}56$, reducing verification budget by
$4\times$. This gap becomes important when verification compute is contended
(\S\ref{sec:serving}).

\textbf{(3) The trained confidence prior adds little.} When added to
\method{} in the sweep (not shown), it improves greedy \tautok{} on 4B by
only ${+}0.03$--$0.07$. The gain decreases to ${+}0.00$--$0.03$ on 14B,
and the prior is harmful under sampling (\S\ref{sec:eval-temp}). The final
method therefore omits it and uses only the calibrated $q$ score
(\S\ref{sec:value}).

\textbf{Significance.} We use a paired per-prompt bootstrap with 10k
resamples. The archived test evaluates the sweep's strongest variant, which
still included the subsequently removed confidence prior whose increment is
at most $0.07$. Against DDTree at $N{=}56$, the gain is ${+}0.75$
\tautok{} on 4B (95\% CI $[+0.69,+0.81]$; wins on $145/150$ prompts) and
${+}0.86$ on 8B ($[+0.79,+0.94]$; $147/150$ wins). The corresponding
prior-free mean margins in Table~\ref{tab:main} are ${+}0.68$ and ${+}0.83$
at $N{=}56$, both substantially larger than the reported interval widths.

\subsection{Adaptive budgets dominate fixed budgets}
\label{sec:eval-adapt}

\begin{table*}[t]
\caption{Calibrated adaptive evaluation at $T{=}0$ on the nine-task suite
(153 held-out prompts per model, three H200 repetitions). Verified is the
mean number of non-root draft nodes per round; AR is speedup over the matched
target-only decoder; Util.\ is the fraction of verified nodes on the
committed path; Uncommitted is $1-\text{Util.}$ Bold marks the highest AR
among displayed rows. Fixed-budget H200 rows are omitted for compactness; the
best omitted AR is $3.463\times$ for $N{=}14$ on 14B, compared with
$3.458\times$ for \method{} $\theta{=}0.05$. Separate Qwen3-4B frontiers
under the 102-prompt protocol appear in Table~\ref{tab:temp2} and
Fig.~\ref{fig:frontier}. Per-dataset \tautok{} on the three additional tasks
appears in Table~\ref{tab:added} (App.~\ref{app:tables}).}
\label{tab:adapt}
\vskip 0.1in
\centering
\scriptsize
\begin{tabular}{llrrrrr}
\toprule
Target & Config & Verified & \tautok{} & AR speedup & Util. & Uncommitted \\
\midrule
\multirow{4}{*}{Qwen3-4B}
& chain $N{=}7$ &  7.00 & 4.961 & 3.124$\times$ & 51.23\% & 48.77\% \\
& \method{} $\theta{=}0.15$ &  6.96 & 5.179 & 3.198$\times$ & 56.92\% & 43.08\% \\
& \method{} $\theta{=}0.05$ & 12.99 & 5.795 & \textbf{3.373$\times$} & 33.93\% & 66.07\% \\
& \method{} $\theta{=}0.02$ & 23.17 & 6.150 & 3.318$\times$ & 20.19\% & 79.81\% \\
\midrule
\multirow{4}{*}{Qwen3-8B}
& chain $N{=}7$ &  7.00 & 5.183 & 3.158$\times$ & 53.13\% & 46.87\% \\
& \method{} $\theta{=}0.15$ &  7.02 & 5.356 & 3.153$\times$ & 58.63\% & 41.37\% \\
& \method{} $\theta{=}0.05$ & 13.44 & 5.979 & \textbf{3.391$\times$} & 33.88\% & 66.12\% \\
& \method{} $\theta{=}0.02$ & 24.30 & 6.286 & 3.278$\times$ & 19.51\% & 80.49\% \\
\midrule
\multirow{4}{*}{Qwen3-14B}
& chain $N{=}7$ &  7.00 & 5.110 & 3.165$\times$ & 52.71\% & 47.29\% \\
& \method{} $\theta{=}0.15$ &  7.20 & 5.368 & 3.273$\times$ & 58.01\% & 41.99\% \\
& \method{} $\theta{=}0.05$ & 13.74 & 5.936 & \textbf{3.458$\times$} & 33.27\% & 66.73\% \\
& \method{} $\theta{=}0.02$ & 24.05 & 6.273 & 3.417$\times$ & 20.00\% & 80.00\% \\
\bottomrule
\end{tabular}

\vskip -0.1in
\end{table*}

Figure~\ref{fig:frontier} evaluates the calibrated stopping rule from
\S\ref{sec:adapt} on Qwen3-4B using 102 held-out prompts, 17 from each of the
original six benchmarks. The prompts are disjoint from calibration, and
decoding is greedy.\footnote{This protocol uses 150 new tokens and
calibration-disjoint prompts, unlike Table~\ref{tab:main}; absolute
\tautok{} values must therefore not be compared across the two evaluations.}
Table~\ref{tab:adapt} extends the evaluation to all nine tasks and all three
model scales, adding wall-clock and node-accounting metrics.

Three observations emerge. \textbf{(1) Strict dominance on the six-task
frontier.} At every operating point in Fig.~\ref{fig:frontier}, the adaptive
policy improves both accepted length and verification cost relative to the
nearest fixed budget (paired bootstrap, 10k resamples). Against $N{=}28$, it
gains ${+}0.073$ \tautok{} (95\% CI $[+0.031,+0.115]$) while verifying
$5.2$ fewer tokens per round (CI $[-6.2,-4.1]$). Against $N{=}14$, it gains
${+}0.058$ ($[+0.018,+0.099]$) while verifying $1.1$ fewer tokens. Against
$N{=}7$, it gains ${+}0.089$ ($[+0.036,+0.145]$) at equal cost. At
$\theta{=}0.02$, the adaptive tree comes within $0.12$ \tautok{} of fixed
$N{=}56$ while reducing verification cost by $59\%$. Per-round tree size
ranges from $7$ nodes at p10 to $38$ at p90, showing that the policy assigns
larger budgets to rounds that can use them.

\textbf{(2) Calibration is essential to the stopping rule.} Using the same
rule with the all-edges-fit calibrator, whose fitting population is biased for
the estimand in Table~\ref{tab:value}, is worse at every $\theta$. The
corrected fit gains ${+}0.14$ to ${+}0.43$ \tautok{} in paired comparisons,
with every CI excluding zero. The difference is largest where the threshold
is most restrictive: at $\theta{=}0.15$, \tautok{} is $5.35$ rather than
$4.92$. Miscalibration also changes the operational meaning of $\theta$: the
same threshold produces a smaller tree and a lower-\tautok{} operating point.
CondAdaptive reports related budget saturation \citep{dominotree2026}. In our
setting, fitting two scalars on the appropriate conditional population
corrects the problem without changing the stopping rule.

\textbf{(3) Reordering alone provides little benefit.} At fixed $N$,
replacing the $\log q$ heap priority with $\log\hat p$ improves \tautok{} by
only ${+}0.03$--$0.04$; the confidence intervals exclude zero at $N{=}14$
and $N{=}56$. The calibrated magnitude matters when absolute values determine
stopping and allocation, but contributes little when only rank is used
(cf.\ \S\ref{sec:value}).

The wall-clock results in Table~\ref{tab:adapt} show the same tradeoff across
scales. \method{} $\theta{=}0.05$ is the fastest adaptive operating point at
each scale, reaching $3.37\times/3.39\times/3.46\times$ AR on 4B/8B/14B,
compared with $3.12\times$--$3.17\times$ for the chain. It is the overall
maximum on 4B and 8B; on 14B, it trails the fixed $N{=}14$ maximum by only
$0.005\times$ ($3.458\times$ versus $3.463\times$), while verifying slightly fewer
nodes and attaining higher \tautok{}. The adaptive operating point commits
$34\%$ of its verified nodes, compared with $17\%$ for fixed $N{=}28$.
Fixed $N{=}56$ has the highest \tautok{} among fixed budgets but falls to
$2.8\times$--$3.0\times$ AR. Thus, \tautok{} and wall-clock rankings do not
coincide: within the adaptive family, $\theta{=}0.05$ provides the best
measured AR point, while $\theta{=}0.02$ provides the best measured
\tautok{} point.

The conditional calibrator remains accurate across model scales, with
held-out ECE $0.0083/0.0141/0.0150$ for 4B/8B/14B. Fixed $N{=}56$, although
best in \tautok{} among the fixed budgets, is slower end to end than the
mid-budget operating points at every scale. This is the
acceptance--verification-cost tradeoff that \tautok{} alone does not capture.

\subsection{Sampling: losslessness and temperature behavior}
\label{sec:eval-temp}

Table~\ref{tab:lossless} provides an empirical check of
Proposition~\ref{prop:lossless}. The sampled tree's TV distance, $0.047$, is
comparable to the finite-sample deviation of direct target sampling,
$0.066$, whereas the deterministic top-$k$ control has TV distance $0.46$.
For example, the control emits a token with target probability $p{=}0.344$
only $6\%$ of the time. We deliberately choose a position with a diffuse
target distribution so that the test is sensitive to this bias.

\begin{table}[t]
\caption{DSpark chain and adaptive \method{} with the final recompute
implementation (A100, six-task held-out suite, 102 prompts, one same-GPU
sequential pass per temperature). AR is wall-clock speedup over the matched
target-only decoder in the same cache-free reference harness (both arms
cache-free, unlike the KV-cached H200 protocol of
Table~\ref{tab:adapt}), which runs at $30.3$ tok/s for $T{=}0$ and
$28.8$ tok/s for $T{>}0$. The adaptive policy uses the batch-size-1 ladder
value: $\theta{=}0.02$ at $T{=}0$ and $\theta{=}0.05$ otherwise. These AR
values are comparable only within this recompute harness. Earlier
stored-proposal fixed-$N$ ladders remain available in the artifacts.}
\label{tab:temp}
\vskip 0.1in
\centering
\scriptsize
\setlength{\tabcolsep}{2.5pt}
\begin{tabular}{lrrrrrr}
\toprule
& \multicolumn{2}{c}{$T{=}0$} & \multicolumn{2}{c}{$T{=}0.5$} & \multicolumn{2}{c}{$T{=}1.0$} \\
\cmidrule(lr){2-3}\cmidrule(lr){4-5}\cmidrule(lr){6-7}
Decoder & \tautok{} & AR & \tautok{} & AR & \tautok{} & AR \\
\midrule
DSpark chain $N{=}7$ & 5.07 & 3.30$\times$ & 5.13 & 3.14$\times$ & 4.97 & 3.06$\times$ \\
\method{} & \textbf{6.34} & \textbf{3.75$\times$} & \textbf{5.92} & \textbf{3.38$\times$} & \textbf{5.72} & \textbf{3.38$\times$} \\
\bottomrule
\end{tabular}
\vskip -0.1in
\end{table}

Table~\ref{tab:temp} reports end-to-end results for the final recompute
implementation. The adaptive tree improves \tautok{} over the DSpark chain
at every temperature: by ${+}25\%$ at $T{=}0$ and by ${+}15\%$ under
sampling. It is also faster in wall-clock time. AR speedup is
$3.75\times$ versus $3.30\times$ at $T{=}0$, $3.38\times$ versus
$3.14\times$ at $T{=}0.5$, and $3.38\times$ versus $3.06\times$ at $T{=}1$,
corresponding to gains of ${+}14\%$, ${+}8\%$, and ${+}10\%$.

The earlier stored-proposal sweep showed the reverse within-harness
ranking --- trees topped \tautok{} while the chain stayed fastest end to
end; \S\ref{sec:eval-temp2} traces this to per-node bookkeeping that the
final recompute verifier eliminates.

The greedy-trained confidence priority is harmful throughout sampled
decoding, so the final \method{} omits it and uses the calibrated Markov
score $\hat p(q_0)$ at every temperature.

\subsection{The unified method at temperature}
\label{sec:eval-temp2}

At matched verification cost, the adaptive tree lies above the fixed-budget
frontier by ${+}0.06$--${+}0.22$ \tautok{} at both sampled temperatures
(Table~\ref{tab:temp2}). The strongest comparisons improve both metrics. At
$T{=}1$, $\theta{=}0.02$ exceeds $N{=}28$ by ${+}0.151$ \tautok{}
(95\% CI $[+0.030,+0.278]$) while verifying $2.45$ fewer tokens per round.
At $T{=}0.5$, the corresponding gain is ${+}0.118$
($[+0.016,+0.222]$) with $3.0$ fewer verified tokens.

Two ablations identify the source of this advantage. First, supplying the
calibrator with temperature-scaled $q_T$ rather than $q_0$ enlarges trees by
$1.2$--$1.9\times$ at the same $\theta$ and reduces the advantage to a tie
with the fixed frontier. Calibration must therefore preserve the absolute
scale used by the stopping rule. Second, the verification rule matters under
sampling. As temperature increases, deterministic-draft exact-match
verification loses effectiveness. In the earlier fixed-budget sweeps at
$T{=}1$ and matched $N{=}56$, sampled trees retain ${+}11\%$ \tautok{} over
exact-match DDTree on the legacy six-task suite and ${+}7\%$ on the nine-task
suite. DominoTree likewise reports that its advantage over DDTree narrows to
ties and losses as $T$ increases \citep{dominotree2026}. Rejection sampling
from the true proposals retains the tree's acceptance advantage, reaching
${+}22\%$ \tautok{} over the chain at $T{=}1$.

The initial stored-proposal implementation kept one full-vocabulary proposal
for every node, adding overhead linear in tree size and leaving the chain
fastest within that harness. The released recompute implementation stores no
per-node proposal vectors. Instead, it reconstructs each visited parent's
proposal bit-for-bit from the backbone output and draw order
(\S\ref{sec:sampling}) and batches tree expansion. Same-GPU interleaved A/B
comparisons within the final harness show that the adaptive tree exceeds the
chain at batch size 1 by ${+}3.0$--$3.4\%$ at $T{=}1$
($\theta{=}0.05$; 7/8 paired wins across two seed sets) and by ${+}6.1\%$ at
$T{=}0.5$ (3/3). Numerical AR values from the stored-proposal and recompute
harnesses are not cross-comparable.

At $\text{bs}{=}16$ under sampling, no branching tree in the $\theta$ sweep
outperforms the chain. The ladder therefore selects its code-identical chain
member, giving parity in those cells (Table~\ref{tab:ladder}). At $T{=}0$,
where verification does not require proposal reconstruction, the adaptive
tree improves both \tautok{} and wall-clock performance at both measured
batch sizes; at $\text{bs}{=}16$, the wall-clock gain is ${+}2.1\%$ with 4/4
paired wins (Tables~\ref{tab:adapt} and~\ref{tab:ladder}).

\subsection{Cost accounting}
\label{sec:cost}

In each round, \method{} performs at most $N$ heap pops. Each pop requires
one $|V|\times r$ matrix--vector product, one softmax, and one top-$K$
operation; at $T{>}0$, child selection instead uses one $K$-draw
without-replacement sample over the vocabulary. With $r{=}256$ and
$|V|{=}152$k, this costs approximately $80$M FLOPs per node, compared with
roughly $8$~GFLOPs for each target-token forward that speculation can avoid.
The drafter still performs one backbone pass regardless of tree shape.
Verification expands the target forward from $7$ to $N$ tokens, which is the
tradeoff measured in Table~\ref{tab:main} and controlled by the scheduler in
\S\ref{sec:serving}. Stage-level timings agree with this accounting: in the
batch-size-1 engine, the complete drafter side, including backbone and
expansion, takes $7$--$16$\,ms per round, compared with $34$--$36$\,ms for
target verification (\S\ref{sec:serving}).

\section{Discussion and Limitations}
\label{sec:limits}

\textbf{Wall-clock scope.} The single-request wall-clock measurements
(Tables~\ref{tab:adapt} and~\ref{tab:temp}) run in a Python research
harness against a matched target-only decoder, not in the
continuous-batching serving implementation of \S\ref{sec:serving}.
Production batched goodput and tail latency therefore remain to be
measured.

\textbf{Block depth and conditioning.} The released drafters use $b{=}7$,
which caps the accepted length of each round. Evaluating deeper blocks
requires retraining, although the drafter training pipeline is public. The
Markov head also conditions only on the immediate parent. On deeper branches,
its predictions therefore approach marginal behavior, consistent with the
measured diminishing returns with depth.

\textbf{Value estimation.} A candidate-token-aware confidence head trained
for this setting is a natural extension of the current two-scalar
calibration. We evaluate calibration across three model scales, held-out
domains, and $T\in\{0,0.5,1\}$ using analytic $q_0$ as input. One zero-shot
mismatch remains: calibration labels use greedy token matches, whereas
verification at $T{>}0$ uses the rejection rule. The temperature results
suggest that this mismatch is small, but we do not isolate it directly.

\textbf{Coverage.} We have not evaluated the Gemma4-12B model pair or
conducted a same-environment comparison with EAGLE-3.

\textbf{Concurrent work.} Conditional tree drafting is developing rapidly,
including DominoTree, PCTree, and JetSpec. The components introduced
here---lossless sampled expansion, acceptance calibration with the correct
fitting population and input, and a budget interface for serving---are
independent of the particular drafter and are designed to apply to models
with inexpensive conditional heads.

\section{Conclusion}
\label{sec:conclusion}

Draft trees can be combined with semi-autoregressive drafters, but they must
use the drafter's conditional predictions: marginal tree constructions fail
on the co-trained backbones studied here. \method{} turns DSpark's Markov
head into a provably lossless tree drafter without additional backbone
passes, improving over DDTree by $9$--$16\%$ at equal verification budget
across three model scales. Per-edge acceptance is accurately estimated by a
two-scalar calibration of the drafter's conditional score, rather than by
its trained confidence head, provided that the fitting population matches
the path-survival estimand. The resulting threshold gives the serving
scheduler a per-round price for tree growth. On Qwen3-4B, across every
evaluated temperature, this rule dominates the fixed-budget
acceptance--verification-cost frontier, making tree shape a controllable
serving variable.

\bibliography{references}

\begin{thebibliography}{25}
\providecommand{\natexlab}[1]{#1}
\providecommand{\url}[1]{\texttt{#1}}
\expandafter\ifx\csname urlstyle\endcsname\relax
  \providecommand{\doi}[1]{doi: #1}\else
  \providecommand{\doi}{doi: \begingroup \urlstyle{rm}\Url}\fi

\bibitem[Cai et~al.(2024)Cai, Li, Geng, Peng, Lee, Chen, and
  Dao]{cai2024medusa}
Cai, T., Li, Y., Geng, Z., Peng, H., Lee, J.~D., Chen, D., and Dao, T.
\newblock Medusa: Simple {LLM} inference acceleration framework with multiple
  decoding heads.
\newblock In \emph{Proceedings of the 41st International Conference on Machine
  Learning (ICML)}, volume 235 of \emph{Proceedings of Machine Learning
  Research}, pp.\  5209--5235. PMLR, 2024.

\bibitem[Chen et~al.(2023)Chen, Borgeaud, Irving, Lespiau, Sifre, and
  Jumper]{chen2023accelerating}
Chen, C., Borgeaud, S., Irving, G., Lespiau, J.-B., Sifre, L., and Jumper, J.
\newblock Accelerating large language model decoding with speculative sampling.
\newblock \emph{arXiv preprint arXiv:2302.01318}, 2023.

\bibitem[Chen et~al.(2026)Chen, Liang, and Liu]{chen2026dflash}
Chen, J., Liang, Y., and Liu, Z.
\newblock {DFlash}: Block diffusion for flash speculative decoding.
\newblock In \emph{Proceedings of the 43rd International Conference on Machine
  Learning (ICML)}, 2026.
\newblock arXiv:2602.06036.

\bibitem[Chen et~al.(2024)Chen, May, Svirschevski, Huang, Ryabinin, Jia, and
  Chen]{chen2024sequoia}
Chen, Z., May, A., Svirschevski, R., Huang, Y., Ryabinin, M., Jia, Z., and
  Chen, B.
\newblock Sequoia: Scalable and robust speculative decoding.
\newblock In \emph{Advances in Neural Information Processing Systems 37
  (NeurIPS)}, 2024.

\bibitem[Cheng et~al.(2026)Cheng, Yu, Shao, Li, Xiong, Qian, Zhu, Ma, Zhang,
  Ye, Chen, Deng, Yu, Dai, Zhang, Wei, Tan, Yang, Xu, Wu, Xu, Wang, Chen, Tian,
  Bi, Hao, Chen, Cao, Zhang, Xu, Zhang, Zhao, and Liang]{dspark2026}
Cheng, X., Yu, X., Shao, C., Li, J., Xiong, Y., Qian, Y., Zhu, J., Ma, S.,
  Zhang, X., Ye, J., Chen, Q., Deng, C., Yu, J., Dai, D., Zhang, Z., Wei, Y.,
  Tan, Y., Yang, W., Xu, R., Wu, Y., Xu, Z., Wang, X., Chen, M., Tian, R., Bi,
  X., Hao, Z., Chen, S., Cao, H., Zhang, W., Xu, A., Zhang, H., Zhao, D., and
  Liang, W.
\newblock {DSpark}: Confidence-scheduled speculative decoding with
  semi-autoregressive generation.
\newblock \emph{arXiv preprint arXiv:2607.05147}, 2026.

\bibitem[Hu et~al.(2026)Hu, Feng, Wu, Yuan, Zhao, Qian, Wang, Zhao, Jiang, Zhu,
  Rosing, and Zhang]{jetspec2026}
Hu, L., Feng, Z., Wu, Y., Yuan, H., Zhao, Y., Qian, Y.-Y., Wang, B., Zhao, P.,
  Jiang, D., Zhu, Y., Rosing, T., and Zhang, H.
\newblock {JetSpec}: Breaking the scaling ceiling of speculative decoding with
  parallel tree drafting.
\newblock \emph{arXiv preprint arXiv:2606.18394}, 2026.

\bibitem[Huang et~al.(2026)Huang, Zhang, Zhang, Lin, Xu, and Zhang]{domino2026}
Huang, J., Zhang, Y., Zhang, Q., Lin, H., Xu, H., and Zhang, L.
\newblock Domino: Decoupling causal modeling from autoregressive drafting in
  speculative decoding.
\newblock \emph{arXiv preprint arXiv:2605.29707}, 2026.

\bibitem[Huang et~al.(2025)Huang, Wu, Shi, Zou, Yu, and Shi]{adaspec2025}
Huang, K., Wu, H., Shi, Z., Zou, H., Yu, M., and Shi, Q.
\newblock {AdaSpec}: Adaptive speculative decoding for fast, {SLO}-aware large
  language model serving.
\newblock In \emph{Proceedings of the 2025 ACM Symposium on Cloud Computing
  (SoCC)}, pp.\  361--374. ACM, 2025.
\newblock \doi{10.1145/3772052.3772239}.

\bibitem[Leviathan et~al.(2023)Leviathan, Kalman, and
  Matias]{leviathan2023fast}
Leviathan, Y., Kalman, M., and Matias, Y.
\newblock Fast inference from transformers via speculative decoding.
\newblock In \emph{Proceedings of the 40th International Conference on Machine
  Learning (ICML)}, volume 202 of \emph{Proceedings of Machine Learning
  Research}, pp.\  19274--19286. PMLR, 2023.

\bibitem[Li et~al.(2024)Li, Wei, Zhang, and Zhang]{li2024eagle2}
Li, Y., Wei, F., Zhang, C., and Zhang, H.
\newblock {EAGLE}-2: Faster inference of language models with dynamic draft
  trees.
\newblock In \emph{Proceedings of the 2024 Conference on Empirical Methods in
  Natural Language Processing (EMNLP)}, pp.\  7421--7432. Association for
  Computational Linguistics, 2024.
\newblock \doi{10.18653/v1/2024.emnlp-main.422}.

\bibitem[Li et~al.(2026)Li, Li, Xie, Song, and Lu]{pctree2026}
Li, Z., Li, T., Xie, C., Song, X., and Lu, H.
\newblock From chains to trees: Parent-conditioned drafting for
  semi-autoregressive speculative decoding.
\newblock \emph{arXiv preprint arXiv:2608.02123}, 2026.

\bibitem[Lin \& Jang(2026)Lin and Jang]{dominotree2026}
Lin, S.~S. and Jang, J.-S.~R.
\newblock {DominoTree}: Conditional tree-structured drafting with domino for
  speculative decoding.
\newblock \emph{arXiv preprint arXiv:2607.08642}, 2026.

\bibitem[Liu et~al.(2025)Liu, Park, Hu, Kwon, Li, Zhang, Du, Mo, You, Cheung,
  Deng, Stoica, and Zhang]{liu2024smartspec}
Liu, X., Park, J., Hu, L., Kwon, W., Li, Z., Zhang, C., Du, K., Mo, X., You,
  K., Cheung, A., Deng, Z., Stoica, I., and Zhang, H.
\newblock Turbospec: Closed-loop speculation control system for optimizing
  {LLM} serving goodput.
\newblock \emph{arXiv preprint arXiv:2406.14066}, 2025.

\bibitem[Liu et~al.(2026)Liu, Yu, Park, Stoica, and Cheung]{liu2026illusion}
Liu, X., Yu, J., Park, J., Stoica, I., and Cheung, A.
\newblock Speculative decoding: Performance or illusion?
\newblock In \emph{Proceedings of Machine Learning and Systems}, volume~8, pp.\
   1719--1741. MLSys, 2026.

\bibitem[Miao et~al.(2024)Miao, Oliaro, Zhang, Cheng, Wang, Zhang, Wong, Zhu,
  Yang, Shi, Shi, Chen, Arfeen, Abhyankar, and Jia]{miao2024specinfer}
Miao, X., Oliaro, G., Zhang, Z., Cheng, X., Wang, Z., Zhang, Z., Wong, R.
  Y.~Y., Zhu, A., Yang, L., Shi, X., Shi, C., Chen, Z., Arfeen, D., Abhyankar,
  R., and Jia, Z.
\newblock Specinfer: Accelerating large language model serving with tree-based
  speculative inference and verification.
\newblock In \emph{Proceedings of the 29th ACM International Conference on
  Architectural Support for Programming Languages and Operating Systems, Volume
  3 (ASPLOS)}, pp.\  932--949. ACM, 2024.
\newblock \doi{10.1145/3620666.3651335}.

\bibitem[Rheinboldt et~al.(2026)Rheinboldt, Berdoz, and
  Wattenhofer]{treeflash2026}
Rheinboldt, P., Berdoz, F., and Wattenhofer, R.
\newblock {TreeFlash}: Parallel {AR}-approximation for faster speculative
  decoding.
\newblock \emph{arXiv preprint arXiv:2606.03819}, 2026.

\bibitem[Ringel \& Romano(2026)Ringel and Romano]{ddtree2026}
Ringel, L. and Romano, Y.
\newblock Accelerating speculative decoding with block diffusion draft trees.
\newblock \emph{arXiv preprint arXiv:2604.12989}, 2026.

\bibitem[Sandler et~al.(2026)Sandler, Christopher, Hartvigsen, and
  Fioretto]{sandler2026specdiff2}
Sandler, J., Christopher, J., Hartvigsen, T., and Fioretto, F.
\newblock {SpecDiff}-2: Scaling diffusion drafter alignment for faster
  speculative decoding.
\newblock In \emph{Proceedings of Machine Learning and Systems}, volume~8, pp.\
   1128--1147. MLSys, 2026.

\bibitem[Shi et~al.(2026)Shi, Xu, Deng, Wu, Liu, Xu, Chen, Zhu, Xu, Huang,
  Yang, and Zhou]{specblock2026}
Shi, W., Xu, Q., Deng, F., Wu, Y., Liu, J., Xu, Y., Chen, H., Zhu, J., Xu, J.,
  Huang, X., Yang, J., and Zhou, X.
\newblock {SpecBlock}: Block-iterative speculative decoding with dynamic tree
  drafting.
\newblock \emph{arXiv preprint arXiv:2605.07243}, 2026.

\bibitem[Wang et~al.(2025)Wang, Su, Li, Xia, Ye, Duan, Wang, and
  Zhang]{wang2024opttree}
Wang, J., Su, Y., Li, J., Xia, Q., Ye, Z., Duan, X., Wang, Z., and Zhang, M.
\newblock {OPT}-tree: Speculative decoding with adaptive draft tree structure.
\newblock \emph{Transactions of the Association for Computational Linguistics},
  13:\penalty0 188--199, 2025.
\newblock \doi{10.1162/tacl_a_00735}.

\bibitem[Wang et~al.(2026{\natexlab{a}})Wang, Chen, Zhen, Liu, Zheng, Liu, Xu,
  and Li]{wang2026prism}
Wang, X., Chen, Y., Zhen, M., Liu, F., Zheng, X., Liu, X., Xu, H., and Li, M.
\newblock {PRISM}: Parametrically refactor inference for speculative decoding
  draft models.
\newblock In \emph{Proceedings of Machine Learning and Systems}, volume~8, pp.\
   2172--2185. MLSys, 2026{\natexlab{a}}.

\bibitem[Wang et~al.(2026{\natexlab{b}})Wang, Huang, and Chen]{taps2026}
Wang, Z., Huang, J., and Chen, X.
\newblock {TAPS}: Target-aware prefix tree selection for diffusion-drafted
  speculative decoding.
\newblock \emph{arXiv preprint arXiv:2606.00487}, 2026{\natexlab{b}}.

\bibitem[Yang \& Li(2026)Yang and Li]{specauf2026}
Yang, T. and Li, M.
\newblock {Spec-AUF}: Accept-until-fail training under train-inference
  misalignment for masked block drafters.
\newblock \emph{arXiv preprint arXiv:2607.01893}, 2026.

\bibitem[Zhang et~al.(2026)Zhang, Qiu, He, and Dai]{caddtree2026}
Zhang, S., Qiu, H., He, H., and Dai, Y.
\newblock Cost-aware diffusion draft trees for speculative decoding.
\newblock \emph{arXiv preprint arXiv:2606.01813}, 2026.

\bibitem[Zhao et~al.(2026)Zhao, Tang, Zhu, Ye, Chang, Lin, Park, Xiao,
  Abdelfattah, Gao, Kasikci, Han, and Stoica]{zhao2026specgen}
Zhao, Y., Tang, J., Zhu, K., Ye, Z., Chang, C.-C., Lin, C., Park, J., Xiao, G.,
  Abdelfattah, M.~S., Gao, M., Kasikci, B., Han, S., and Stoica, I.
\newblock Accelerating large-scale reasoning model inference with sparse
  self-speculative decoding.
\newblock In \emph{Proceedings of Machine Learning and Systems}, volume~8, pp.\
   251--265. MLSys, 2026.

\end{thebibliography}
\bibliographystyle{mlsys2026}

\appendix

\section{Block-Semantics Alignment Probe}
\label{app:probe}

We first obtain greedy continuations from the target. For each candidate
block semantics, we then draft one block per round and verify the resulting
greedy chain. This isolated H200 probe uses 48 prompts, eight from each
original task, and generates 150 new tokens. Its metric is the number of
committed draft tokens divided by the number of target-verification rounds.
On \texttt{dspark\_qwen3\_4b\_block7}, LM-shifted decoding yields $3.492$
accepted draft tokens per round, compared with $1.275$ under the in-place
interpretation. On \texttt{dflash\_qwen3\_4b\_block7}, the corresponding
values are $3.043$ and $0.015$. Under the incorrect in-place interpretation,
the DSpark Markov head retains enough bigram-level signal to obscure the
alignment error; DFlash is nearly degenerate. Every result in this paper uses
the verified LM-shifted semantics.

\section{Additional Tables and Figures}
\label{app:tables}

Tables~\ref{tab:temp2}--\ref{tab:lossless} and Figure~\ref{fig:frontier}
collect the deferred results referenced from the main text: the $T{>}0$
fixed-versus-\method{} grid, edge-acceptance prediction, per-dataset
\tautok{} on the added tasks, the modeled goodput grid, the
acceptance--cost frontier, and the empirical losslessness check.

\begin{table}[t]
\caption{Fixed and adaptive variants of the method in
\S\ref{sec:expansion}--\S\ref{sec:adapt} at $T>0$ (Qwen3-4B, 102 held-out
prompts). Each entry reports verified draft tokens per round / \tautok{}.
When calibration uses $q_0$, adaptive stopping lies above the matched fixed
frontier at both temperatures. Bold identifies the adaptive member of each
matched-cost comparison, not the column maximum.}
\label{tab:temp2}
\vskip 0.1in
\centering
\small
\begin{tabular}{lcc}
\toprule
Config & $T{=}0.5$ & $T{=}1.0$ \\
\midrule
chain ($N{=}7$)          &  7.0 / 5.13 &  7.0 / 4.97 \\
fixed $N{=}14$           & 13.6 / 5.87 & 14.0 / 5.59 \\
fixed $N{=}28$           & 26.4 / 6.20 & 28.0 / 5.92 \\
fixed $N{=}56$           & 51.1 / 6.52 & 56.0 / 6.28 \\
\method{} $\theta{=}0.15$ &  7.1 / \textbf{5.36} &  7.5 / \textbf{5.08} \\
\method{} $\theta{=}0.05$ & 13.2 / \textbf{5.93} & 14.6 / \textbf{5.73} \\
\method{} $\theta{=}0.02$ & 23.3 / \textbf{6.32} & 25.5 / \textbf{6.08} \\
\bottomrule
\end{tabular}
\vskip -0.1in
\end{table}

\begin{table}[t]
\caption{Acceptance prediction on 34{,}216 Qwen3-4B tree edges. Full-pool
rows are descriptive; held-out rows fit even-indexed prompts and evaluate
odd-indexed prompts. $^\dagger$ restricts evaluation to the 6{,}256
ancestors-accepted edges. The AUC change from $0.883$ to $0.948$ comes from
this population change, not Platt scaling.}
\label{tab:value}
\vskip 0.1in
\centering
\scriptsize
\begin{tabular}{lcc}
\toprule
Acceptance predictor & AUC & ECE \\
\midrule
\multicolumn{3}{l}{\emph{Full pool (descriptive)}} \\
confidence head $c$ (raw)              & 0.699 & 0.260 \\
Markov conditional $q_0$ (raw)         & 0.883 & 0.131 \\
\addlinespace
\multicolumn{3}{l}{\emph{Held-out odd-indexed prompts}} \\
$c$ $+$ Platt (all-edge fit)            & 0.702 & 0.070 \\
$q_0$ (raw)$^\dagger$                  & \textbf{0.948} & 0.042 \\
$q_0$ $+$ Platt (all-edge fit)$^\dagger$ & \textbf{0.948} & 0.059 \\
$q_0$ $+$ Platt (conditional fit)$^\dagger$
                                        & \textbf{0.948} & \textbf{0.0105} \\
\bottomrule
\end{tabular}
\vskip -0.1in
\end{table}

\begin{figure}[t]
\centering
\includegraphics[width=\columnwidth]{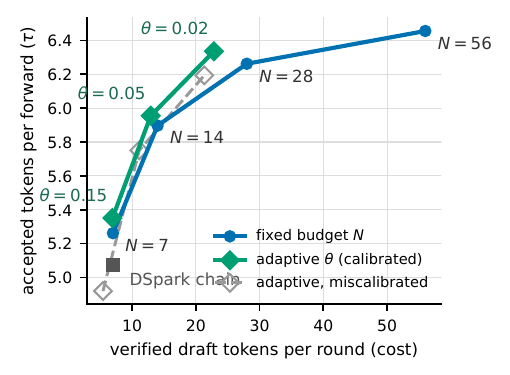}
\caption{Acceptance--cost frontier (Qwen3-4B, 102 held-out prompts, greedy).
Calibrated adaptive stopping (green) lies above the fixed-budget line (blue)
at every operating point. The same stopping rule with the miscalibrated
all-edges fit (gray, dashed) falls below it, with the largest gap at the most
restrictive threshold.}
\label{fig:frontier}
\end{figure}

\begin{table}[t]
\caption{Per-dataset \tautok{} on the three additional tasks for
\method{} $\theta{=}0.05$ in the three-scale evaluation.}
\label{tab:added}
\vskip 0.1in
\centering
\scriptsize
\begin{tabular}{lccc}
\toprule
Target & AIME25 & LiveCodeBench & Arena-Hard-v2 \\
\midrule
Qwen3-4B  & 6.398 & 6.005 & 4.337 \\
Qwen3-8B  & 6.756 & 6.279 & 4.460 \\
Qwen3-14B & 6.617 & 6.547 & 4.282 \\
\bottomrule
\end{tabular}
\vskip -0.1in
\end{table}

\begin{table}[t]
\caption{Modeled $T{=}0$ goodput (tokens/s) from measured step costs and
acceptance. Adaptive is the per-load oracle envelope over
$\theta\in\{0.02,0.05,0.15\}$, shown as an upper frontier rather than the
shipped ladder at every intermediate load; bold marks the row maximum.}
\label{tab:goodput}
\vskip 0.1in
\centering
\small
\begin{tabular}{rcccc}
\toprule
bs & no-spec & chain $N{=}7$ & fixed $N{=}56$ & adaptive \\
\midrule
 1 &   31 &  140 & \textbf{182} & 176 \\
 4 &  123 &  563 & \textbf{745} & 683 \\
 8 &  247 & 1134 & 1034 & \textbf{1425} \\
16 &  495 & 2034 & 1231 & \textbf{2153} \\
32 & 1000 & 2555 & 1312 & \textbf{2718} \\
64 & 1524 & 2798 & 1342 & \textbf{2984} \\
\bottomrule
\end{tabular}
\vskip -0.1in
\end{table}

\begin{table}[t]
\caption{Empirical losslessness at $T{=}1.0$: total variation from the
analytic target at a high-entropy position over 300 one-round trials. Direct
target sampling is the finite-sample reference.}
\label{tab:lossless}
\vskip 0.1in
\centering
\small
\begin{tabular}{lc}
\toprule
Estimator & TV to analytic $p$ \\
\midrule
Direct target sampling (reference)      & 0.066 \\
\method{} sampled tree (ours)        & 0.047 \\
Top-$k$ + naive rejection (control)  & 0.462 \\
\bottomrule
\end{tabular}
\vskip -0.1in
\end{table}

\section{Reproducibility}
\label{app:repro}

All drafter checkpoints and evaluation prompts are available through
DeepSpec. We validate every greedy configuration against direct decoding;
the linear tree ($K{=}1$) also reproduces chain drafting token-for-token.
Code is publicly available at
\url{https://github.com/PopSoda2002/TreeSpark}; the raw experiment
artifacts, the scripts that produced them, and the step-by-step
reproduction manual are attached to the same repository's v0.1.0 release
as an archive, and every reported result records its GPU metadata.

\end{document}